\documentclass{article}
\pdfoutput=1

     \usepackage[nonatbib, preprint]{neurips_2019}

\usepackage[utf8]{inputenc} 
\usepackage[T1]{fontenc}    
\usepackage{hyperref}       
\usepackage{url}            
\usepackage{booktabs}       
\usepackage{amsfonts}       
\usepackage{nicefrac}       
\usepackage{microtype}      

\PassOptionsToPackage{noabbrev, capitalise}{cleveref}

\input{abbrv}
\input{thmdef}

\usepackage{standalone}
\usepackage{tikz, pgfplots}
\pgfplotsset{compat=1.14}
\newlength\figureheight
\newlength\figurewidth
\usepackage{enumitem}
\setlist[description]{font=\normalfont\quad}

\DeclareMathOperator{\smallUnd}{\bigwedge}

\title{A Rate-Distortion Framework for Explaining\\Neural Network Decisions}

\author{%
  Jan Macdonald,\: Stephan W{\"a}ldchen,\: Sascha Hauch \\
  Technische Universit{\"a}t Berlin\\
  \texttt{\{macdonald, stephanw, hauch\}@math.tu-berlin.de} \\
  \And
  Gitta Kutyniok \\
  Technische Universit{\"a}t Berlin\\
  University of Troms{\o}\\
  \texttt{kutyniok@math.tu-berlin.de} 
}

\begin{document}

\maketitle

\begin{abstract}
  We formalise the widespread idea of interpreting neural network decisions as an explicit optimisation problem in a rate-distortion framework. A set of input features is deemed relevant for a classification decision if the expected classifier score remains nearly constant when randomising the remaining features. We discuss the computational complexity of finding small sets of relevant features and show that the problem is complete for $\SNP^{\SPP}$, an important class of computational problems frequently arising in AI tasks. Furthermore, we show that it even remains $\SNP$-hard to only approximate the optimal solution to within any non-trivial approximation factor. 
  Finally, we consider a continuous problem relaxation and develop a heuristic solution strategy based on assumed density filtering for deep ReLU neural networks. We present numerical experiments for two image classification data sets where we outperform established methods in particular for sparse explanations of neural network decisions.
\end{abstract}

\section{Introduction}\label{sec:introduction}

Traditional machine learning models such as linear regression, decision trees, or $k$-nearest neighbours allow for a straight-forward human interpretation of the model prediction.
In contrast, the reasoning of highly nonlinear and parameter-rich neural networks remains generally inaccessible.
Recent years have seen progress on this front with the introduction of multiple explanation models for deep neural networks \cite{bach-plos15, NIPS2017_7062shap, marcotr2016lime, ShrikumarGK17deeplift, simonyan2013deep, zeiler2014visualizing}. These models provide additional information to a prediction by assigning importance values to individual input features .
The evaluation of the soundness of these methods has so far mostly been limited to comparison with human intuition as to which features are important. 
Notable exceptions are Shapley values \cite{shapley1952gamevalues} that draw justification from game theoretic aspects, as well as the proposition of pixel-flipping to numerically compare explanation methods \cite{WojBGM2017pixelflip}.
However, there is yet no formal definition of what it means for an input feature to be relevant for a classification decision. 

In this paper we introduce a rigorous approach to obtain interpretable neural network decisions.
More precisely, in \cref{sec:ratedistortion} we formulate the problem of determining the most relevant components of an input signal for a classifier prediction as an optimisation problem in a rate-distortion framework. We show in a worst-case analysis in \cref{sec:complexity} that this problem is generally hard to solve and to approximate, which justifies the use of heuristic methods. In \cref{sec:heuristic} we propose a problem relaxation together with a heuristic solution strategy for deep feed forward neural networks. Finally, we present numerical experiments and compare different explanation methods for two image classification data sets in \cref{sec:experiments}.

\paragraph{Notation}
Throughout, $d\in\N$ is the dimension of the signal domain, $\bfx\in[0,1]^d$ is an arbitrary fixed input signal, and $\Phi\colon[0,1]^d\rightarrow[0,1]$ is a classifier function for a signal class $\mathcal{C}\subseteq[0,1]^d$. The function $\Phi$ can for example be described by a neural network. The classification score $\Phi(\bfx)$ represents the classifiers prediction on how likely it is that $\bfx$ belongs to the class $\mathcal{C}$. We set $[d]=\{1,\dots,d\}$ and for a subset $S\subseteq[d]$ denote by $\bfx_S$ the restriction of $\bfx$ to components indexed by $S$.
Finally, $\bfone_d\in\R^d$ is a vector of all ones, $\diag\kl{\bfx}$ the diagonal matrix with entries given by $\bfx$, and $\odot$ (resp. $\oslash$) the component-wise Hadamard product (resp. quotient) of vectors or matrices of the same dimensions. We write $\bfx^2 = \bfx\odot\bfx$ and consider univariate functions applied to vectors to act component-wise.

\section{Rate-Distortion viewpoint}\label{sec:ratedistortion}
We formulate the task of explaining the classifier prediction $\Phi(\bfx)$ as finding a partition of the components of $\bfx\in\ekl{0,1}^d$ into a subset $S\subseteq[d]$ of \emph{relevant} components and its complement $S^c$ of \emph{non-relevant} components. The partition should be chosen in a way such that fixing the relevant components already determines the classifier output for almost all possible assignments to the non-relevant ones. More precisely, let $\CV$ be a probability distribution on $[0,1]^d$ and $\bfn\sim\CV$ a random vector. We define the \emph{obfuscation} of $\bfx$ with respect to $S$ and $\CV$ as a random vector $\bfy$ that is deterministically defined on $S$ as $\bfy_S=\bfx_S$ and distributed on the complement according to $\bfy_{S^c}=\bfn_{S^c}$. We write $\CV_S$ for the resulting distribution of $\bfy$. Having only knowledge about the signal on $S$ and filling in the rest of the signal randomly will mostly preserve the class prediction if $\bfx_S$ contains the information that was relevant for the classifier decision. We measure the change in the classifier prediction using the squared distance. The expected distortion of $S$ with respect to $\Phi$, $\bfx$, and $\CV$ is defined as 
\begin{equation*}\label{eq:dist} 
    D(S) = D(S,\Phi,\bfx,\CV) = \E_{\bfy \sim \CV_{S}}\left[\frac{1}{2}\left(\Phi(\bfx)-\Phi(\bfy)\right)^2\right].
\end{equation*}

We naturally arrive at a rate-distortion trade-off that intuitively gives us a measure of relevance. The terminology is borrowed from information theory where rate-distortion is used to analyse lossy data compression. In that sense, the set of relevant components is a compressed description of the signal and the expected deviation from the classification score is a measure for the reconstruction error. We define the rate-distortion function as
\begin{equation}\label{eq:rate}
    R(\epsilon) = \min_{} \skl{ \,\bkl{S} \,:\, S \subseteq [d],\, D(S) \leq \epsilon\,}.
\end{equation}
The smallest set $S$ that ensures a limited distortion will be composed of the most relevant input components. This rate-distortion function will also allow us later to compare the performance of different explanation models. \par 

We now want to discuss the difficulty of finding such a set. Note, that the trivial choice of setting $S=[d]$ ensures zero distortion. We show that for distortion limits greater than zero one cannot systematically find a set of relevant components that is significantly smaller than the trivial set. This specifically holds when $\Phi$ is represented by a neural networks which is the case we are particularly interested in. We derive our hardness results for the special case of Boolean circuits representable by ReLU neural networks of moderate size.\footnote{The depth can be chosen constant and the total number of neurons polynomial in $d$.}

\section{Complexity theoretic analysis}\label{sec:complexity}

Let us for now consider the special case of binary input signals $\bfx \in \skl{0,1}^d$ and classifier functions $\Phi: \skl{0,1}^d \rightarrow \skl{0,1}$ described as Boolean circuits, as well as the uniform distribution over binary vectors, i.e. $\CV = \mathcal{U}\kl{\skl{0,1}^d}$. \par 
\subsection{Discrete problem formulation}
We call a subset $S\subseteq\ekl{d}$ a $\delta$\emph{-relevant} set for $\Phi$ and $\bfx$ if
$\P_{\bfy} \kl{ \Phi(\bfy) = \Phi(\bfx)\,\middle|\,\bfy_S = \bfx_S} \geq \delta$.
We can assert that a set $S$ is $\delta$-relevant if and only if it has limited distortion $D(S) \leq \frac{1-\delta}{2}$. Therefore calculating the rate-distortion function is essentially the same task as finding minimal relevant sets.
\begin{definition}
For $\delta\in\ekl{0,1}$ we define the \textsc{Relevant-Input} problem as follows. \\
\textsc{Given:} $\Phi\colon\skl{0,1}^d\to \skl{0,1}$, $\bfx\in\skl{0,1}^d$, and $k\in [d]$.\\
\textsc{Decide:} Does there exist an $S \subseteq [d]$ with $\bkl{S} \leq k$ such that $S$ is $\delta$-relevant for $\Phi$ and $\bfx$?
\end{definition}
The optimisation version associated to this decision problem can be defined in the standard way.
\begin{definition}
For $\delta\in\ekl{0,1}$ we define the \textsc{Min-Relevant-Input} problem as follows. \\
\textsc{Given:} $\Phi\colon\skl{0,1}^d\to \skl{0,1}$ and $\bfx\in\skl{0,1}^d$.\\
\textsc{Minimise:} $k \in [d]$  such that there exists an $S \subseteq [d]$ with $\bkl{S} \leq k$ that is $\delta$-relevant for $\Phi$ and $\bfx$.
\end{definition}
Then the following hardness result holds \cite{wmhk-2019-rel-compl}.
\begin{theorem}\label{thm:nppp-complete}
 For $\nicefrac{1}{2}\leq \delta <1$ the \textsc{Relevant-Input} problem is $\SNP^{\SPP}$-complete.
\end{theorem}
The class $\SNP^{\SPP}$ is the class of all problems decidable by a non-deterministic Turing machine with access to an oracle for problems in $\SPP$ \cite{gill1977computational}. 
The class $\SNP^\SPP$ appears frequently in artificial intelligence tasks such as optimisation under uncertainty \cite{littman1998computational} and is assumed significantly harder than either $\SNP$ or $\SPP$ \cite{dechter1998bucket, park2002map}. This means also that solving \textsc{Min-Relevant-Input} is an extremely difficult problem. In practice it would often suffice to solve it only approximately and one might hope that efficient approximation algorithms exist. Hence, a more practically relevant result is the following. 
\begin{theorem}\label{thm:inapprox}
  Assume $\SP\ne \SNP$ then for any $\alpha\in(0,1)$ there is no polynomial time approximation algorithm for \textsc{Min-Relevant-Input} with an approximation factor of $d^{1-\alpha}$.
\end{theorem}
We give a full proof in the following section and want to remark that this is a simplification of an even a slightly stronger result shown in \cite{wmhk-2019-rel-compl}. 

\subsection{Inapproximability}
An approximation algorithm for \textsc{Min-Relevant-Input} has an approximation factor $c\geq 1$ if it finds an approximate solution $k$ such that $k^\ast\leq k \leq ck^\ast$ for all problem instances, where $k^\ast$ denotes the respective exact solution. Choosing the trivial solution $S=\ekl{d}$, thus considering all components as relevant, results in a factor $d$. \Cref{thm:inapprox} states that it is generally hard to achieve better factors.

The proof proceeds in two steps. First, we introduce a gapped version of the decision problem and show that it is $\SNP$-hard. Second, we show that the gapped problem would be in $\SP$ if there exists a good polynomial time approximation algorithm for \textsc{Min-Relevant-Input}. 

\begin{definition}
For $\delta \in\ekl{0,1}$ we define the \textsc{Auxilliary Problem} (AP) as follows. \\
\textsc{Given}: $\Phi\colon\skl{0,1}^d\to\skl{0,1}$, $\bfx\in\skl{0,1}^d$, and $k,m\in\N$, $1\leq k\leq m\leq d$.\\
\textsc{Decide}: Which of the two options (if any) holds:
  \begin{description}[labelwidth=\widthof{\emph{Yes}-instances:\qquad}]
     \item[\emph{Yes}-instances:] There exists $S\subseteq\ekl{d}$ with $|S|\leq k$ and $S$ is $\delta $-relevant for $\Phi$ and $\bfx$.
     \item[\emph{No}-instances:] All $S\subseteq\ekl{d}$ with $|S|\leq m$ are not $\delta$-relevant for $\Phi$ and $\bfx$.
\end{description}
\end{definition}
The restriction to the case $k=m$ is exactly the \textsc{Relevant-Input} problem. However here we also allow the case $k<m$ with a gap in the set sizes. It might happen that none of the two options hold. In this case we consider any answer to be acceptable.

\begin{lemma}\label{lem:ap-np-hard}
  For $\delta\in(0,1)$ we have $\textsc{SAT}\preceq_p\textsc{AP}$, in particular in this case \textsc{AP} is $\SNP$-hard.
\end{lemma}
\begin{proof}
  Let $\Phi\colon\skl{0,1}^d\to\skl{0,1}$ be a \textsc{SAT} instance. We will construct $\skl{\Phi^\prime, \bfx^\prime, k^\prime, m^\prime}$ that is a \emph{Yes}-instance for \textsc{AP} if and only if $\Phi$ is a \emph{Yes}-instance for \textsc{SAT}. Let $q=\left\lceil\log_2\kl{d}-\log_2\kl{1-\delta}\right\rceil$ and $p=\left\lfloor-\log_2\kl{\delta}\right\rfloor+1$.
  We set $k^\prime = dq$, $m^\prime \geq k^\prime$ arbitrary but at most polynomial in $d$. Also let $\Phi^\prime\colon\skl{0,1}^{d\times q}\times\skl{0,1}^{m^\prime+p}\to\skl{0,1}$ be given by
  \[
    \Phi^\prime(\bfu^{(1)},\dots,\bfu^{(q)},\bfv)= \Phi\kl{\smallUnd_{j=1}^{q}\bfu^{(j)}}\odr\kl{\smallUnd_{i=1}^{m^\prime+p} v_i},
  \]
  where each $\bfu^{(j)}\in\skl{0,1}^d$ and the conjunction within $\Phi$ is component-wise, and set $\bfx^\prime = \bfone_{dq+m^\prime+p}$. This is a polynomial time construction. By the choice of $\Phi^\prime$ and $\bfx^\prime$ we guarantee $\Phi^\prime(\bfx^\prime)=1$ regardless of the satisfiability of $\Phi$. In the following we denote $ \mathbf{U}=(\bfu^{(1)},\dots,\bfu^{(q)})$.
  
  \paragraph{Necessity:} Let $\Phi$ be a \emph{Yes}-instance for \textsc{SAT}. Thus, there exists $\bfx\in\skl{0,1}^d$ with $\Phi(\bfx)=1$. Let $S=\skl{\,i\in [d]\,:\, x_i=1\,}$ and $S^\prime = S \times [q]$. Denote $A(\bfu^{(1)},\dots,\bfu^{(q)}) = \Und_{(i,j)\in S^\prime} u^{(j)}_i$. Clearly, $\bkl{S^\prime} \leq k^\prime$. Further, $S^\prime$ is $\delta$-relevant for $\Phi^\prime$ and $\bfx^\prime$ if $\P_{(\bfU,\bfv)}\kl{\Phi^\prime(\bfU,\bfv)\,\middle|\, A(\bfU)}\geq \delta$. We have
  \begin{align*}
    \P_{(\bfU,\bfv)}\kl{\Phi^\prime(\bfU,\bfv)\,\middle|\, A(\bfU)} &\geq \P_{\bfU}\kl{\smallUnd_{j=1}^{q} \bfu^{(j)}=\bfx\,\middle|\, A(\bfU)}.
  \end{align*}
  From this, with a union bound, we obtain
  \begin{equation*}
      \P_{\bfU}\kl{\smallUnd_{j=1}^{q} \bfu^{(j)}=\bfx \,\middle|\, A(\bfU)} = 1-\P_\bfU\kl{\exists i\in S^c\,:\,\smallUnd_{j=1}^{q} u^{(j)}_i}
      \geq 1 - |S^c|2^{-q} \geq \delta,
  \end{equation*}
  which shows that $\skl{\Phi^\prime,\bfx^\prime,k^\prime,m^\prime}$ is a \emph{Yes}-instance for \textsc{AP}.

  \paragraph{Sufficiency:} Now conversely let $\Phi$ be a \emph{No}-instance for \textsc{SAT}. Then for any $S^\prime\subseteq\ekl{dq+m^\prime+p}$ with $|S^\prime|\leq m^\prime$ we have
  \begin{align*}
  \P_\bfy\kl{\Phi^\prime(\bfy)=\Phi^\prime(\bfx^\prime)\,\middle|\, \bfy_{S^\prime}=\bfx^\prime_{S^\prime}} &= \P_{(\bfU,\bfv)}\kl{\smallUnd_{i=1}^{m^\prime+p} v_i\,\middle|\, (\bfU,\bfv)_{S^\prime}=\bfone} \\
  &\leq 2^{-(m^\prime+p-|S^\prime|)} < \delta, 
  \end{align*}
  showing that $S^\prime$ is not $\delta$-relevant for $\Phi^\prime$ and $\bfx^\prime$. Hence, $\skl{\Phi^\prime,\bfx^\prime,k^\prime,m^\prime}$ is a \emph{No}-instance for \textsc{AP}.
\end{proof}

Finally, we can prove the inapproximability of the \textsc{Min-Relevant-Input} problem.
\begin{proof}[Proof of \cref{thm:inapprox}]
    We show that the existence of a polynomial time approximation algorithm for \textsc{Min-Relevant-Input} with approximation factor $d^{1-\alpha}$ would allow us to decide \textsc{AP} in polynomial time for certain instances. These can be chosen as in the proof of \cref{lem:ap-np-hard}, which in turn implies that we could decide \textsc{SAT} in polynomial time. This is only possible if $\SP=\SNP$. 
    
    Let $\Phi\colon\skl{0,1}^d\to\skl{0,1}$ be a \textsc{SAT} instance and $\skl{\Phi^\prime,\bfx^\prime,k^\prime,m^\prime}$ an equivalent \textsc{AP} instance as in the proof of \cref{lem:ap-np-hard}. We have seen that there is some freedom in the choice of $m^\prime$ as long as it satisfies $k^\prime\leq m^\prime$ and is at most polynomial in $d$.
    We choose $ m'=\left\lceil \max(2k'({k'}^{1-\alpha}+p^{1-\alpha}), (2k')^{\frac{1}{\alpha}}+1) \right\rceil $ with $p=\left\lfloor-\log_2\kl{\delta}\right\rfloor+1$ as before. Recall that $k^\prime = dq$ with $q=\left\lceil\log_2\kl{d}-\log_2\kl{1-\delta}\right\rceil$, so clearly $m^\prime$ is polynomial in $d$ and $k^\prime\leq m^\prime$. Further, we have $ m'> (2k')^{\frac{1}{\alpha}} $ so $ 1-k'{m'}^{-\alpha}>\frac{1}{2} $ and therefore
    \[ m'(1-k'{m'}^{-\alpha})>\frac{m'}{2}\ge k'({k'}^{1-\alpha}+p^{1-\alpha}). \]
    Now let $d^\prime=k^\prime+m^\prime+p$ denote the number of variables of $\Phi^\prime$. By the subadditivity of the map $z\mapsto z^{1-\alpha}$, we finally obtain
    \[
        k^\prime {d^\prime}^{1-\alpha} = k^\prime(k^\prime+m^\prime+p)^{1-\alpha}
        \leq k^\prime\left({k^\prime}^{1-\alpha} + {m^\prime}^{1-\alpha} + p^{1-\alpha}\right)
        < m^\prime.
    \]
    It remains to show that an \textsc{AP} instance with $m^\prime > k^\prime {d^\prime}^{1-\alpha}$ can be decided by an approximation algorithm for \textsc{Min-Relevant-Input} with approximation factor ${d^\prime}^{1-\alpha}$. Assume such an algorithm exists and let $k$ be a solution produced by the algorithm. Then for the true optimal solution $k^\ast$ we have $k^\ast\leq k\leq {d^\prime}^{1-\alpha}k^\ast$. 
    
    Firstly, assume that $\skl{\Phi^\prime, \bfx^\prime, k^\prime, m^\prime}$ is a \emph{Yes}-instance for \textsc{AP}. Then there exists a $\delta$-relevant set of size $k^\prime$. However, notice that no set smaller than $k^\ast$ can be $\delta$-relevant. This implies $k^\ast\leq k^\prime$ and therefore $k\leq{d^\prime}^{1-\alpha}k^\prime <m^\prime$.
    
    Secondly, assume that $\skl{\Phi^\prime, \bfx^\prime, k^\prime, m^\prime}$ is a \emph{No}-instance for \textsc{AP}. Then all sets of size at most $m^\prime$ are not $\delta$-relevant. But there exists a $\delta$-relevant set of size $k^\ast$. This implies $k\geq k^\ast>m^\prime$.
    
    Altogether, checking whether $k<m^\prime$ or $k>m^\prime$ decides $\skl{\Phi^\prime,\bfx^\prime,k^\prime, m^\prime}$.
\end{proof}

\Cref{thm:inapprox} shows that no efficient approximation algorithm for \textsc{Min-Relevant-Input} exists (unless $\SP = \SNP$). Either we have to resort to heuristics, or introduce stronger restrictions on the problem. We choose the former and present a general heuristic for neural networks. To this end, we also further relax the problem formulation to a continuous setting.

\section{Problem relaxation and solution heuristic}\label{sec:heuristic}
The problem class  $\SNP^{\SPP}$ already gives a hint of what difficulties we have to overcome. First, we need an efficient way to judge whether a chosen set leads to small expected distortion. This amounts to calculating expectation values. Secondly, we need to optimise over all feasible sets --- a combinatorial optimisation problem.
We propose a heuristic solution for both hurdles when the classifier $\Phi$ is a deep neural network. 

\subsection{Neural network functions}
Let $L\in\N$ denote the number of layers of the neural network, $d_1, \dots, d_{L-1}\in\N$ and $d_0=d$, $d_L=1$. Further let $\kl{\bfW_1, \bfb_1}, \dots, \kl{\bfW_L, \bfb_L}$ with $\bfW_i\in\R^{d_{i}\times d_{i-1}}$, $\bfb_i\in\R^{d_i}$ for $i\in[L]$ be the weight matrices and bias vectors of an $L$-layer feed forward neural network. We then consider functions of the form
\[\Phi(\bfx) = \bfW_L\varrho(\bfW_{L-1}\varrho(\dots \varrho(\bfW_1\bfx+\bfb_1) \dots ) + \bfb_{L-1}) + \bfb_L.\]
In the following we consider the rectified linear unit (ReLU) activation function $\varrho(x)=\max\skl{0,x}$.

\subsection{Continuous problem relaxation}
To address the combinatorial optimisation problem we make use of the following relaxation. 
Instead of binary relevance decisions (\emph{relevant} versus \emph{non-relevant}) encoded by the set $S$, we allow for a continuous relevance score for each component, encoded by a vector $\bfs\in[0,1]^d$.
We redefine the obfuscation of $\bfx$ with respect to $\bfs$ as a component-wise convex combination
\begin{equation}\label{eq:contobfuscation}
    \bfy = \bfx \odot \bfs + \bfn \odot (\bfone - \bfs)
\end{equation}
of $\bfx$ and $\bfn\sim\CV$. As before we write $\CV_{\bfs}$ for the resulting distribution of $\bfy$. This is a generalisation of the obfuscation introduced in \cref{sec:ratedistortion} which can be recovered by choosing $\bfs$ equal to one on $S$ and zero on $S^c$. The natural relaxation of the set size $|S|$ is the norm $\|\bfs\|_1 = \sum_{i=1}^{d} |s_i|$. As before we define the expected distortion and rewrite it in its bias-variance decomposition
\begin{equation} \label{eq:bias-variance}
  D(\bfs) 
  = \E_{\bfy \sim \CV_{\bfs}}\left[\frac{1}{2}\left(\Phi(\bfx)-\Phi(\bfy)\right)^2 \right]
  = \frac{1}{2}\left(\Phi(\bfx)-\E_{\bfy\sim\CV_\bfs}\ekl{\Phi(\bfy)}\right)^2 + \frac{1}{2}\V_{\bfy\sim\CV_\bfs}\ekl{\Phi(\bfy)},
\end{equation}
where $\V$ denotes the covariance matrix. The expected distortion is determined by the first and second moment of the output layer distribution.
The exact calculation of expectation values and variances for arbitrary functions is in itself already a hard problem. One possibility to overcome this issue is to approximate the expectation by a sample mean. Depending on the dimension $d$ and the distribution $\CV$ sampling might be infeasible. Thus, we focus on a second possibility, which takes the specific structure of $\Phi$ more into account. From \eqref{eq:contobfuscation} it is straight-forward to obtain the first and second moment 
\[\E\ekl{\bfy} = \bfx\odot\bfs + \E\ekl{\bfn}\odot(\bfone-\bfs)\quad\text{and}\quad\V\ekl{\bfy}=\diag(\bfone-\bfs)\V\ekl{\bfn}\diag(\bfone-\bfs)\] 
of the input distribution $\CV_\bfs$. It remains to transfer the moments from the input to the output layer. This is discussed in \cref{sec:adf}. Instead of the hard constraint $D(\bfs)\leq\epsilon$ as in \eqref{eq:rate} we formulate the continuous rate minimisation problem in its Lagrangian formulation
\begin{equation}\label{eq:contopt}
    \minimize  \quad D(\bfs) + \lambda \|\bfs\|_1 \quad \subjectto \quad \bfs\in[0,1]^d
\end{equation}
with a regularisation parameter $\lambda>0$. We call this approach to obtaining relevance scores for classifier decisions RDE (Rate-Distortion Explanation). Depending on the activation function, the distortion does not need to be differentiable. However, the ReLU activation is differentiable almost everywhere. As commonly done during the training of neural networks, we simply use (projected) gradient descent to find a stationary point of \eqref{eq:contopt}. \par  

\subsection{Assumed density filtering}\label{sec:adf}
To address the challenge of efficiently approximating the expectation values in \eqref{eq:bias-variance} we utilise the layered structure of $\Phi$ and propagate the distribution of the neuron activations through the network. 
For this, we use an approximate method, called assumed density filtering (ADF), see for example \cite{Minka:2001:FAA:935427, Boyen:1998:TIC:2074094.2074099}, which has recently been used for ReLU neural networks in the context of uncertainty quantification \cite{gast-roth-8578453}. 
In a nutshell, at each layer we assume a Gaussian distribution for the input, transform it according to the layers weights $\bfW$, biases $\bfb$, and activation function $\varrho$, and project the output back to the nearest Gaussian distribution (w.r.t.\ KL-divergence). This amounts to matching the first two moments of the distribution \cite{Minka:2001:FAA:935427}. We now state the ADF rules for a single network layer. Applying these repeatedly gives us a way to propagate moments through all layers and obtain an explicit expression for the distortion. \par 

Let $\bfz\sim\CN\kl{\bfmu,\bfSigma}$ be normally distributed with some mean $\bfmu$ and covariance $\bfSigma$. An affine linear transformation preserves Gaussianity and acts on the mean and covariance in the well-known way
 \begin{equation}\label{eq:affinevar}
  \E\ekl{\bfW\bfz+\bfb} = \bfW \bfmu + \bfb\quad\text{and}\quad
  \V\ekl{\bfW\bfz+\bfb} = \bfW \bfSigma \bfW^\ast. 
 \end{equation}
The ReLU non-linearity $\varrho$ presents a difficulty as it changes a Gaussian distribution into a non-Gaussian one.  Let $f$ and $F$ be the probability density and cumulative distribution function of the univariate standard normal distribution, let $\bfsigma$ be the vector of the diagonal entries of $\bfSigma$, and $\bfeta = \bfmu \oslash \bfsigma$. Then as in \cite[Equation 10a]{gast-roth-8578453} we obtain
\begin{equation*}
     \E\ekl{\varrho(\bfz)} = \bfsigma \odot f(\bfeta) + \bfmu \odot  F(\bfeta).
\end{equation*}
Unfortunately the off-diagonal terms of the covariance matrix of $\varrho(\bfz)$ are thought to have no closed form solution \cite{fayed2014evaluation}. We either make the additional assumption that the network activations within each layer are uncorrelated, which amounts to propagating only the diagonal $\V_{\textrm{diag}}$ of the covariance matrices through the network and simplifies \eqref{eq:affinevar} to
\begin{equation*}
    \V_{\textrm{diag}}\ekl{\bfW\bfz+\bfb} = \kl{\bfW\odot\bfW} \bfsigma, 
\end{equation*}
and results, as also seen in \cite[Equation 10b]{gast-roth-8578453}, in 
\begin{equation*}
\V_{\textrm{diag}}\ekl{\varrho(\bfz)} = \bfmu\odot\bfsigma\odot f(\bfeta) + \kl{\bfsigma^2 + \bfmu^2} \odot F(\bfeta) - \E\ekl{\varrho(\bfz)}^2.
\end{equation*}
Or, we use an approximation for the full covariance matrix
\begin{equation}
 \V\ekl{\varrho(\bfz)} \approx \bfN \bfSigma \bfN,  \label{eq:multi-reluvar}
\end{equation}
with $\bfN = \diag\kl{ F(\bfeta) }$. This ensures positive semi-definiteness and symmetry. Depending on the network size it is usually infeasible to compute the full covariance matrix at each layer. However, if we choose a symmetric low-rank approximation factorisation $\V\ekl{\bfy}\approx \bfQ\bfQ^\top$ at the input layer with $\bfQ\in\R^{d\times r}$ for $r\ll d$ (for example half of a truncated singular value decomposition), then the symmetric update \eqref{eq:multi-reluvar} allows us to propagate only one of the factors through the layers. The full covariance is then only recovered at the output layer. This immensely reduces the computational cost and memory requirement. 

Altogether, combining the affine linear with the non-linear transformation, tells us how to propagate the first two moments through a ReLU neural network in the ADF framework. We investigate both the \emph{diagonal} as well as the \emph{low-rank} approximation to the covariance matrix in our numerical inquiry.

\section{Numerical experiments}\label{sec:experiments}

We present numerical experiments for interpretable neural networks comparing our proposed method RDE to several state-of-the-art methods. We generate relevance mappings for greyscale images of handwritten digits from the MNIST dataset \cite{mnist726791} as well as color images from the STL-10 dataset \cite{coates2011stl}.

As reference distribution $\CV$ we consider a Gaussian distribution with mean and variance or low-rank factorisation of the covariance matrix estimated from the training data set as described in \cite{chan1982samplevariance, rehurek2011low-rank-variance}. The low-rank factorisation is obtained from a truncated singular value decomposition with rank $r=30$. During the gradient descent optimisation we use a momentum term with factor $0.85$ and determine the step size according to a backtracked Armijo linesearch \cite{nocedal2006optimisation}. We initialise with the constant map $\bfs= 0.2 \cdot \bfone_d$ and set the regularisation parameter to $\lambda=0.5$ for MNIST and $\lambda=0.05$ for STL-10. The parameters were tuned by grid-search using only one image per data set.

We compare our method to Layer-wise Relevance Propagation (LRP) \cite{bach-plos15}, Deep Taylor decompositions \cite{MONTAVON20181}, Sensitivity Analysis \cite{simonyan2013deep}, SmoothGrad \cite{smilkov2017smoothgrad}, Guided Backprop \cite{springenberg2014guidedbackprop}, SHAP \cite{NIPS2017_7062shap}, and LIME \cite{marcotr2016lime}. For this we use the Innvestigate\footnote{\url{https://github.com/albermax/innvestigate}} \cite{DBLP:journals/corr/abs-1808-04260}, SHAP\footnote{\url{https://github.com/slundberg/shap}}, and LIME\footnote{\url{https://github.com/marcotcr/lime}} toolboxes.

Different interpretation approaches produce differently scaled and normalised relevance mappings. Some methods generate non-negative mappings corresponding to the importance \emph{for} the classifier score, whereas other methods also generate negative relevance that can be interpreted as speaking \emph{against} a classifier decision. This has to be dealt with carefully to allow for a fair comparison.

We propose a variant of the \emph{relevance ordering}-based test introduced in \cite{WojBGM2017pixelflip}. Each relevance mapping induces an ordering of the input signal components by sorting them according to their relevance score (breaking ties randomly). We start with a completely random signal, replace increasingly large parts of it by the original input, and observe the change in the classifier score. This is then averaged over multiple random input samples. A good relevance mapping will lead to a fast convergence to the score of the original signal when the most relevant components are fixed first. In other words the distortion quickly drops to zero. The described process allows to us approximately evaluate the rate-distortion function.

The focus of this paper is the interpretability of neural networks, not on their training so we will keep the description of the training process quite brief.

\paragraph{MNIST experiment} We trained a convolutional neural network (three convolution layers each followed by average-pooling and finally two fully-connected layers and softmax output) end-to-end up to a test accuracy of 0.99. The standard training/validation/testing split of the data was used.

The relevance mappings for one example image of the digit six are shown in \cref{fig:mnist-mappings}. The mappings are calculated for the pre-softmax score of the class with the highest activation. Both variants of our proposed method generate similar results and highlight an area at the top that distinguishes the digit six from, for example, the digits zero and eight. The relevance-ordering test results are shown in \cref{fig:combined-rate} (left). We observe that the expected distortion drops fastest for our proposed method indicating that the most relevant components were correctly identified.

\begin{figure}
    
  \centering\scriptsize
  \begin{tabular}{c@{\;\;}c@{\;\;}c@{\;\;}c@{\;\;}c}
  
    image & SmoothGrad \cite{smilkov2017smoothgrad} & LRP-$\alpha$-$\beta$ \cite{bach-plos15} & SHAP \cite{NIPS2017_7062shap} & RDE (diagonal) \\
    \includegraphics[height=2.0cm]{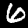} &
    \includegraphics[height=2.0cm]{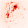} &
    \includegraphics[height=2.0cm]{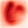} &
    \includegraphics[height=2.0cm]{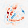} &
    \includegraphics[height=2.0cm]{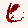} \\[.5em]
    
    Sensitivity \cite{simonyan2013deep} & Guided Backprop \cite{springenberg2014guidedbackprop} & Deep Taylor \cite{MONTAVON20181} & LIME \cite{marcotr2016lime} & RDE (low-rank) \\
    \includegraphics[height=2.0cm]{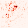} &
    \includegraphics[height=2.0cm]{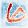} &
    \includegraphics[height=2.0cm]{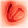} &
    \includegraphics[height=2.0cm]{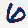} &
    \includegraphics[height=2.0cm]{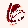} 
  
  \end{tabular}

    \caption{Relevance mappings generated by several methods for an image from the MNIST dataset classified as \emph{digit six} by our network. The colourmap indicates positive relevances as red and negative relevances as blue.}
    \label{fig:mnist-mappings}
\end{figure}


\paragraph{STL-10 experiment} We use a VGG-16 network \cite{DBLP:journals/corr/SimonyanZ14a} pretrained on the Imagenet dataset and retrain on the STL-10 dataset in three stages. First, we train only the fully-connected layers for 500 epochs, then end-to-end for another 500 epochs, and finally after replacing all max-pooling layers by average-pooling layers again end-to-end for another 500 epochs to a final test accuracy of $0.935$.

The relevance mappings for one example image of a dog are shown in \cref{fig:stl-mappings}. As before the mappings are calculated for the pre-softmax score of the class with the highest activation. The difference between both variants of our proposed method is more pronounced here. The low-rank variant captures finer details, for example in the dog's face. The relevance-ordering test results are shown in \cref{fig:combined-rate} (right). We observe that although our method does not obtain the smallest expected distortions across all rates it has the fastest dropping distortion for low rates indicating that the most relevant components were correctly identified. This is to be expected as our method generates sparse relevance scores highlighting only few of the most relevant components and is not designed for high rates.

\begin{figure}
    
  \centering\scriptsize
  \begin{tabular}{c@{\;\;}c@{\;\;}c@{\;\;}c@{\;\;}c}
  
    image & SmoothGrad \cite{smilkov2017smoothgrad} & LRP-$\alpha$-$\beta$ \cite{bach-plos15} & SHAP \cite{NIPS2017_7062shap} & RDE (diagonal) \\
    \includegraphics[height=2.0cm]{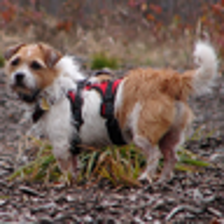} &
    \includegraphics[height=2.0cm]{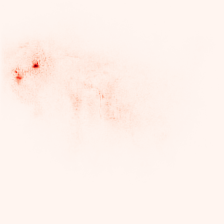} &
    \includegraphics[height=2.0cm]{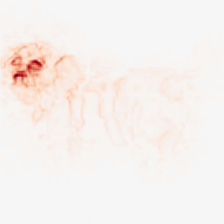} &
    \includegraphics[height=2.0cm]{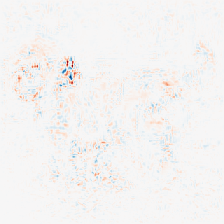} &
    \includegraphics[height=2.0cm]{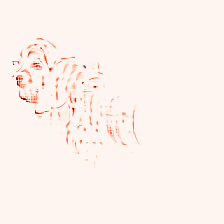} \\[.5em]
    
    Sensitivity \cite{simonyan2013deep} & Guided Backprop \cite{springenberg2014guidedbackprop} & Deep Taylor \cite{MONTAVON20181} & LIME \cite{marcotr2016lime} & RDE (low-rank) \\
    \includegraphics[height=2.0cm]{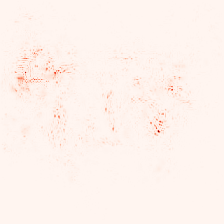} &
    \includegraphics[height=2.0cm]{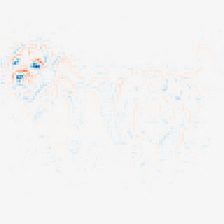} &
    \includegraphics[height=2.0cm]{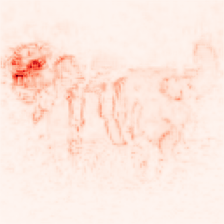} &
    \includegraphics[height=2.0cm]{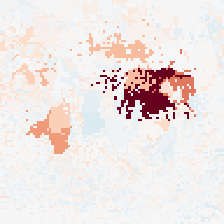} &
    \includegraphics[height=2.0cm]{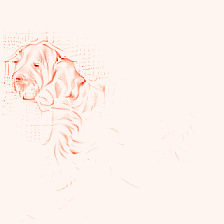}
  
  \end{tabular}

    \caption{Relevance mappings generated by several methods for an image from the STL-10 dataset classified as \emph{dog} by our network. The colourmap indicates positive relevances as red and negative relevances as blue.}
    \label{fig:stl-mappings}
\end{figure}


\begin{figure}
    \input{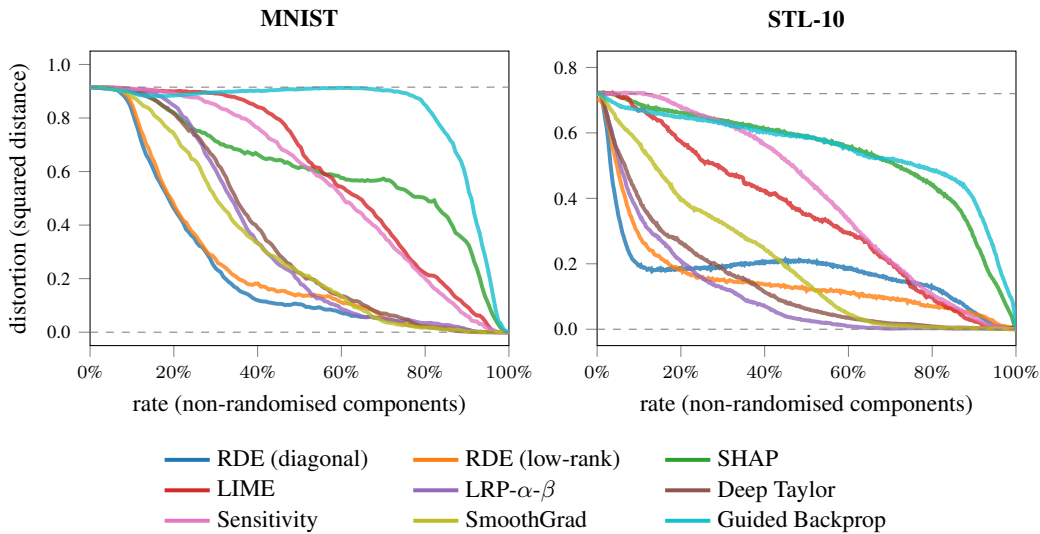}
    \caption{Relevance ordering test results (approximate rate-distortion function) of several methods for the MNIST (left) and STL-10 (right) dataset. An average result over 50 images from the respective test set (5 images per class randomly selected) and 512 (MNIST) and 64 (STL-10) random input samples per image is shown.}
    \label{fig:combined-rate}
\end{figure}

\section{Conclusion}
We introduced a formal rate-distortion framework for explaining classification decisions and analysed the computational complexity of the arising optimisation problem. We saw that in the worst case it is hard to solve and even hard to approximate, which justifies the use of heuristic explanation methods in practical applications. We proposed our method RDE (Rate-Distortion Explanation), which involves a minimisation procedure. We compared it numerically to previous methods and observed that it performs best in the following sense: it captures the smallest set of relevant components, which leads to the steepest curve in the rate-distortion function for small rates.

\subsubsection*{Acknowledgments}

The authors would like to thank Philipp Petersen for several fruitful discussions during the early stage of the project. J. M. and S. W. acknowledge support by DFG-GRK-2260 (BIOQIC). S. H. is grateful for support by CRC/TR 109 ``Discretization in Geometry and Dynamics''.  G. K. acknowledges partial support by the Bundesministerium f{\"u}r Bildung und Forschung (BMBF) through the ``Berliner Zentrum f{\"u}r Machinelles Lernen'' (BZML), by the Deutsche Forschungsgemeinschaft (DFG) through Grants CRC 1114 ``Scaling Cascades in Complex Systems'', CRC/TR 109 ``Discretization in Geometry and Dynamics'', DFG-GRK-2433 (DAEDALUS), DFG-GRK-2260 (BIOQIC), SPP 1798 ``Compressed Sensing in Information Processing'' (CoSIP), by the Berlin Mathematics Research Centre MATH+, and by the Einstein Foundation Berlin.

\bibliographystyle{abbrv}
\bibliography{references}

\end{document}